\documentclass{ws-aa}

\usepackage{subcaption}
\usepackage{amsmath, amsfonts, amssymb, bm}
\usepackage{etoolbox}
\usepackage{times}
\usepackage{xcolor}
\usepackage{enumitem}
\usepackage{url} 
\usepackage{graphicx}
\usepackage{float}
\usepackage[font=small,skip=0pt]{caption}
\usepackage{booktabs}
\usepackage[sort,compress]{cite}
\usepackage[verbose]{hyperref}
\hypersetup{colorlinks=false,allbordercolors=blue,pdfborderstyle={/S/U/W 1}}

\newcommand{\X}{\mathcal{X}}
\newcommand{\Hs}{\mathcal{H}}

\newcommand{\I}{\mathbb{I}}
\newcommand{\CalL}{\mathbb{L}}

\newcommand{\A}{\mathbb{A}}
\newcommand{\Real}{\mathbb{R}}
\newcommand{\EXP}{\mathbb{E}}

\newcommand{\norm}[1]{\left\lVert#1\right\rVert}
\newcommand{\innerpro}[1]{\left\langle#1\right\rangle}
\newcommand{\tG}{\widetilde{G}}
\newcommand{\tg}{\widetilde{g}}
\newcommand{\abs}[1]{\left|#1\right|}
\newcommand{\blam}{\bm{\lambda}}

\begin{document}

\markboth{E. Gizewski, M. Holzleitner, L. Mayer-Suess, S. Pereverzyev Jr., S. Pereverzyev}{Multiparameter regularization and aggregation in the context of polynomial functional regression}

\catchline{}{}{}{}{}

\title{Multiparameter regularization and aggregation in the context of polynomial functional regression}

\author{Elke R. Gizewski}
\address{Department of Radiology, Medical University of Innsbruck, Anichstrasse 35 \\ 6020, Innsbruck, Austria \\
Neuroimaging Research Core Facility, Medical University of Innsbruck, Anichstrasse 35\\ 6020, Innsbruck, Austria\\
\email{elke.gizewski@i-med.ac.at}}

\author{Markus Holzleitner\footnote{Corresponding author}}
\address{Institute for Machine Learning, Johannes Kepler University, Altenberger Straße 69 \\4040, Linz, Austria\\
\email{holzleitner@ml.jku.at}}

\author{Lukas Mayer-Suess}
\address{Department of Neurology, Medical University of Innsbruck, Anichstrasse 35 \\ 6020, Innsbruck, Austria\\
\email{lukas.mayer@i-med.ac.at}}

\author{Sergiy Pereverzyev Jr. }
\address{Department of Radiology, Medical University of Innsbruck, Anichstrasse 35 \\ 6020, Innsbruck, Austria \\
Neuroimaging Research Core Facility, Medical University of Innsbruck, Anichstrasse 35\\ 6020, Innsbruck, Austria\\
\email{sergiy.pereverzyev@i-med.ac.at}}

\author{Sergei V. Pereverzyev}
\address{Johann Radon Institute for Computational and Applied Mathematics, Austrian Academy of Sciences, Altenberger Straße 69\\ 4040, Linz, Austria\\
\email{sergei.pereverzyev@oeaw.ac.at}}

\maketitle


\begin{abstract}
Most of the recent results in polynomial functional regression have been focused on an in-depth exploration of single-parameter regularization schemes. In contrast, in this study we go beyond that framework  by introducing an algorithm for multiple parameter regularization and presenting a theoretically grounded method for dealing with the associated parameters. This method facilitates the aggregation of models with varying regularization parameters. The efficacy of the proposed approach is assessed through evaluations on both synthetic and some real-world medical data, revealing promising results.
\end{abstract}



\section{Introduction} \label{sec:intro}

Functional data analysis has emerged as a vibrant and dynamic research area and is present in various aspects of our daily lives, such as climate studies, medicine, economics, and healthcare, just to name a few. Typically, functional data appear in the forms
of time series, shapes, images, and analogous objects. While the term "functional data analysis" was first used in \cite{ramsay1982data,ramsay1991some}, significant advancements have happened since then. For a comprehensive exploration of methods, theory, and applications, we refer to seminal review articles like  \cite{ramsay2002applied, wang2016functional, kokoszka2017introduction, reiss2017methods}, and also to the very recently appeared special issue \cite{aneiros2022functional}.

This work focuses specifically on functional data inputs that are labeled with scalar-valued outputs. One of the most extensively studied methods in this context assumes a linear relationship between inputs and outputs, so that the outputs can be represented as linear functionals of the (functional) inputs, accompanied maybe by an additional noise term. One popular approach to capture linear functional regression is based on reproducing kernel Hilbert space (RKHS) techniques, so that the known arguments from kernel regression (see e.g. \cite{caponnetto2007optimal, guo2017learning, lu2020balancing, lin2017distributed}) can be used. A by no means complete list of works in this direction can be found in \cite{tong2018analysis, tong2021distributed,yuan2010reproducing} and references therein. We may also mention \cite{holzleitner2023domain}, where linear functional regression approach is proposed in a more sophisticated setting of domain generalization.

Similarly to the case of extending standard linear regression by allowing polynomial interactions, polynomial functional regression (PFR), which includes the functional linear model and functional quadratic model as two special cases,  was proposed in \cite{muller2010quadratic}.
Then it has been discussed in \cite{tao2014polynomial, reiss2017methods} and recently in \cite{holzleitner2023regularized}, where a complete treatment of the interplay between smoothness, capacity and general one parameter regularization schemes is provided (as done e.g. for standard kernel regression in \cite{lu2020balancing, guo2017learning}). In particular, the study \cite{holzleitner2023regularized} has advocated the use of iterated one-parameter Tikhonov regularization method in the context of PFR.

One drawback of using single-parameter (iterated) Tikhonov regularization is, that all norms of the individual monomials in the regularization term are given equal weight, and therefore the advantage of using higher order monomials is not developed to its fullest potential. Consequently, it is advisable to introduce specific weight parameters for each individual monomial, and we envision it as a good place to advertise multi-parameter regularization in this context. 

In general, multi-parameter (MP) regularization schemes have a rich history, both in terms of theory and applications, and we refer to \cite{lu2013regularization}[Chapter 3] and references therein for a comprehensive summary. It is interesting to note that the usage of multiple parameters has been judged variously by different authors. Just to give two examples: in \cite{xu2006multiple}, the authors found, that it provides only marginal improvements, whereas in \cite{belkin2006manifold} it is claimed that MP-regularization helps significantly, when the one parameter counterparts do not lead to satisfying results. One main finding of our work is, that in the case of PFR, we are in a similar situation as in \cite{belkin2006manifold}, and one can demonstrate the advantage of using multi-parameter PFR in numerical examples based on synthetic toy data and on some real-world medical data. 

At the same time, there is a common belief that the choice of the regularization parameters is crucial, and we are only aware of a few works that tackle this serious challenge in the MP case: a heuristic L-curve based strategy is proposed in \cite{belge2002efficient}, in \cite{chen2008multi, bauer2006utilization, bauer2007regularization} knowledge of the noise structure is required and in \cite{lu2011multi} an approach based on the discrepancy principle is discussed, which is costly to compute. 

The solution that we will propose in the specific setting of PFR, is based on the so-called aggregation by the linear functional strategy,
which may be traced back to \cite{chen2015aggregation} (see also Section 3.5 \cite{pereverzyev2022introduction} and references therein).
In the context of standard scalar and vector valued regression such type of aggregation has recently lead to successful performances in domain adaptation, a field that in many aspects is very sensitive to parameter selection as well, see e.g. \cite{dinu2023addressing, gizewski2022regularization}. 
However, we are not aware of any works that employ the aggregation techniques in the context of functional data with MP regularization yet, and thus another main part of our study is to provide theoretical and numerical evidence, that aggregation can be successfully applied in these settings as well. 

The main findings of this work can therefore be summarized as follows:
\begin{itemize}
\item We introduce multi-parameter regularization in the context of PFR and derive a linear system that allows us to compute the corresponding solutions.
\item In order to deal with turning of multiple 
regularization parameters, we propose an aggregation procedure 
in the context of PFR.
\item We provide numerical evidence, that MP regularization and aggregation can be useful concepts for PFR also in practice, on the one hand on synthetic data, and on the other hand, on data from a medical application, where the task is to detect stenosis in brain arteries.
\end{itemize}

Our work will now be structured as follows. In Section 2, we will recall the setting of regularized PFR, by repeating the definitions, assumptions and estimates from \cite{holzleitner2023regularized}. In Section 3, an algorithm how to compute the solution associated to MP-PFR is discussed, whereas Section 4 proposes an aggregation strategy for PFR, which can be evaluated numerically and which comes with additional theoretical guarantees. Section 5 is then devoted to the experiments on synthetic and real-world medical data.
\section{Setting}

\subsection{Overall setting and assumptions} \label{subsec:notation}

Let $\mathbb{I} \subset \mathbb{R}^d$ and consider the associated space $L^2\left(\mathbb{I}\right)$ consisting of square integrable functions with respect to the Lebesgue measure $\mu$, so that
\begin{align*}
\norm{u}^2_{L^2(\mathbb{I})}=\int_{\mathbb{I}} |u(t)|^2 d \mu(t).
\end{align*}
Moreover, let $L^2(\Omega, \mathbb{P})$ be a space of random variables $Y=Y(\omega)$ defined on a probability space $(\Omega, \mathcal{F},\mathbb{P})$, $\omega \in \Omega$, with bounded second moments, 
so that
\begin{align*}
\norm{Y}^2_{L^2(\Omega, \mathbb{P})}:= \mathbb{E}|Y|^2 = \int_{\Omega} |Y(\omega)|^2 d\mathbb{P}(\omega).
\end{align*}
Consider also the tensor product $L^2(\Omega, \mathbb{P}) \otimes L^2\left(\mathbb{I}\right)$, 
which is nothing but a collection of random variables $X(\omega,s)$ indexed by points $s \in \mathbb{I}$ and having bounded second moments in the following sense:
\begin{align*} 
\norm{X}_{\mathbb{P}, \mu}^2:=\mathbb{E} \norm{X(\omega,\cdot )}^2_{L^2(\mathbb{I})}.
\end{align*}
The inner products in the considered Hilbert spaces $\Hs$ will always be denoted by $\innerpro{.,.}_{\Hs}$, and the space is indicated by a subscript.

Functional data consist of random i.i.d. samples of functions $X_1(s),...,X_N(s)$, that can be seen as realizations of a stochastic process $X(\omega,s) \in L^2(\Omega, \mathbb{P}) \otimes L^2\left(\mathbb{I}\right)$. Now let us discuss the setting of polynomial functional regression (PFR): Let $Y \in L^2(\Omega, \mathbb{P}) $ be a scalar response, and $X \in L^2(\Omega, \mathbb{P}) \otimes L^2\left(\mathbb{I}\right)$ be the corresponding functional predictor. We make the following assumption on $X$ (as imposed in a similar way, e.g., in \cite{yuan2010reproducing,tong2018analysis}):
\begin{assumption} \label{ass:unif}
\begin{align*}  
\sup_{\omega \in \Omega} \norm{X(\omega,\cdot )}_{L^2(\I)} \le \kappa.
\end{align*}
\end{assumption}
In PFR one aims at minimizing the expected prediction risk: 
\begin{align} \label{eq:pftfr}
    \mathcal{E}(U_p(X))=\mathbb{E} \left( |Y(\omega )-U_p(X(\omega,\cdot )) |^2 \right) \to \min,
\end{align}
where $U_p(X(\omega,\cdot ))$ is a polynomial regression of order $p$:
\begin{align*}
U_p(X(\omega, \cdot))=u_0+\sum_{l=1}^p \int_{\I^l} u_l(s_1,...,s_l) \prod_{j=1}^l  X(\omega,s_j) d\mu(s_j).
\end{align*}
Here $u_0 \in L^2_0 :=\Real$, and $u_l \in L^2_l$, where
\begin{align*}
L^2_l=\underbrace{L^2(\I) \otimes \cdots \otimes L^2(\I)}_{l\text{ -times}}.
\end{align*}
To proceed and formalize the setting further, consider the operator
\begin{align*}
A_0: \Real \to L^2(\Omega, \mathbb{P}) 
\end{align*}
assigning to any $u_0 \in \Real$ the corresponding constant random variable. Moreover, consider $A_l: L^2_l \to  L^2(\Omega, \mathbb{P}) $, such that
\begin{align} \label{eq:aldef}
(A_l u)(\omega)= \int_{\I^l} u_l(s_1,...,s_l) \prod_{j=1}^l  X(\omega,s_j) d\mu(s_j).
\end{align}
Let, also,  $\CalL^2=\bigoplus_{l=0}^p L^2_l$ be a direct sum of spaces $L^2_l$ consisting of finite sequences $u=(u_0,...,u_p)$, 
$u_l \in L^2_l$, $l=0,1,...,p$, equipped with the norm $\norm{u}_{\CalL^2}^2=\sum_{l=0}^p \norm{u_l}_{L^2_l}^2$, and consider the bounded linear operator (which is also a Hilbert-Schmidt one, as will be seen from Lemma \ref{lem:hs_bounds}) $\A:\CalL^2 \to L^2(\Omega, \mathbb{P})$, given by
\begin{align} \label{eq:CalApdef}
\A u=(A_0,A_1,...,A_p) \circ (u_0,u_1,...,u_p)=\sum_{l=0}^p
A_l u_l. 
\end{align}

Observe that for any $u \in  L^2(\Omega, \mathbb{P})$ the operator $A_l^*: L^2(\Omega, \mathbb{P}) \to L^2_l$ assigns to it the element
\begin{align*}
(A_l^* u)(s_1,...s_l)=\int_{\Omega} u(\omega) \prod_{i=1}^l  X(\omega,s_i) d\mathbb{P}(\omega),
\end{align*}
and therefore, $\A^* \A$ is a $(p+1) \times (p+1)$ matrix of the operators 
\begin{align*}
\A^* \A= \left\{ A_k^* A_l: L^2_l \to L^2_k, k,l=0,1,...,p \right\}
\end{align*}
, where $A_0^{*} A_0u_0 =u_0$ and
\begin{align} 
A_0^{*} A_lu&=\int_{\Omega} \int_{\I^l} u(s_1,...,s_l) \prod_{i=1}^l X(\omega,s_i) d\mu(s_i) d \mathbb{P} (\omega), \nonumber\\
A_k^{*} A_lu(s_1,...,s_k)&=\int_{\Omega} \prod_{j=1}^k X(\omega,s_j) \int_{\I^l} u(\tilde{s}_1,...,\tilde{s}_l) \prod_{i=1}^l X(\omega,\tilde{s}_i) d\mu(\tilde{s}_i)  d \mathbb{P} (\omega),  \nonumber\\
k,l&=1,...,p.  \nonumber 
\end{align}

Equipped with this notation, we can write that
$U_p(X(\omega, \cdot))=\A u$, such that \eqref{eq:pftfr} is reduced to the least square solution of the equation $\A u=Y$, because $\mathcal{E}(U_p(X))=\norm{Y-\A u}^2_{L^2(\Omega, \mathbb{P})}$. Let us also use the following standard assumption:
\begin{assumption} \label{ass:projection}
The projection $\mathcal{P} Y$ of $Y$ on the closure of the range of $\A$ is such that $\mathcal{P} Y \in \text{Range}(\A)$.
\end{assumption}
It is well known (see, e.g.,\cite{lu2013regularization}[Proposition 2.1.]), that under Assumption \ref{ass:projection} the minimizer $u = u^+=(u_0^+,...,u_p^+)$ of \eqref{eq:pftfr} solves the normal equation 

\begin{align} \label{eq:normal}
\A^* \A u=\A^* Y.
\end{align}
Let us elaborate in more detail on the intuition and significance of this assumption in the context of our work.
\begin{remark}
In the theory of linear inverse problems, a distinction is made between solutions to the equation $\A u = Y$ and solutions to the normal equation \eqref{eq:normal}, which correspond to minimizers of the associated least squares functional $\norm{\A u - Y}_{L^2(\Omega, \mathbb{P})}$. When $Y$ belongs to the range of $\A$, every solution of the original equation is also a solution to \eqref{eq:normal}. Assumption \ref{ass:projection} guarantees that \eqref{eq:normal} remains solvable even in more general situations where $Y \notin \operatorname{Range}(\A)$, provided that $\mathcal{P} Y \in \operatorname{Range}(\A)$. This operator-based formulation allows us to exploit various probabilistic and functional analytic techniques, as discussed in the subsequent sections of the manuscript, which would otherwise not be applicable. For further details, we refer to \cite[Proposition 2.1, Remark 2.1]{lu2013regularization}. This assumption is therefore fundamental for our analysis, although relaxing it could represent an interesting avenue for future research.
\end{remark}

We will also use the fact that for any $u \in \CalL^2$
\begin{align} \label{eq:norm_comparisons1}
\norm{\A u}_{L^2(\Omega, \mathbb{P})}=\norm{\sqrt{\A^* \A} u}_{ \CalL^2},
\end{align}
which follows immediately from the polar decomposition of the operator $\A$.

For the further analysis let us also adopt the following response noise model:
\begin{assumption} \label{as:noise_model}
\begin{align} \label{eq:noise_model}
Y= \A u^+ + \varepsilon,
\end{align}
 where a noise variable $\varepsilon: \Omega \to \Real$ is independent from $X$, $\mathbb{E}(\varepsilon)=0$, and for some $\sigma>0$ it should satisfy either the condition
\begin{align} \label{eq:noise_1}
\mathbb{E}(|\varepsilon(\omega)|^2) \le \sigma^2,
\end{align}
 
 or obey, for any integer $\tilde{m} \geq 2$ and some $M>0$ , a slightly stronger moment condition, which is also standard in the literature, see e.g. \cite{tong2021distributed}, 
\begin{align} \label{eq:noise_2}
\mathbb{E}(|\varepsilon(\omega)|^{\tilde{m}}) \le  \frac12 \sigma^2 \tilde{m}! M^{\tilde{m}-2}.
\end{align}
\end{assumption}
However, the involved operators are inaccessible, because we do not know $\mathbb{P}$. Thus, we want to approximate them by using training data $(Y_i, X_i(\cdot))$, $i=1,...,N$, consisting of $N$ independent samples
of the response and the functional predictor $(Y(\omega), X(\omega, \cdot))$, so that
\begin{align*}
Y_i=\A_i u^+ +\varepsilon_i,
\end{align*}
where $\A_i$ is defined in the same way as $\A$ by the replacement of $X(\omega, \cdot)$
in the formulas \eqref{eq:aldef} and \eqref{eq:CalApdef} with $X_i(\cdot)$, and $\varepsilon_i$ is a sample from the noise variable introduced in Assumption \ref{as:noise_model}.

Moreover, $u^+$ does not depend continuously on the initial datum, such that we need to employ a regularization.

The simplest and arguably most well known regularization in this context is the single-parameter Tikhonov regularization, so for $\lambda>0$ we want to find the minimizer $u_{\lambda}$ of the regularized PFR
\begin{align} \label{eq:tikpftfr}
\norm{Y-\A u}^2_{L^2(\Omega, \mathbb{P})} + \lambda \norm{u}^2_{\mathbb{L}^2} \to \text{min},
\end{align} 
which solves the equation $\lambda u + \A^* \A u= \A^* Y$ and can be approximated by the solution $u_{\lambda}^N$ of 
\begin{align} \label{eq:tikpftfremp}
\lambda u + [\A^* \A]_N u= [\A^* Y]_N.
\end{align}
These approximations are given by $[\A^* \A]_N=\left\{ [A_k^* A_l]_N: L^2_l \to L^2_k, k,l=0,1,...,p \right\}$
so that:
\begin{align} 
[A_0^{*} A_0]_N u&=u,  \nonumber \\
[A_0^{*} A_l]_N u&=\frac1N \sum_{i=1}^N \int_{\I^l} u(s_1,...,s_l) \prod_{j=1}^l X_i(s_j) d\mu(s_j), \nonumber \\
[A_k^{*} A_l]_N u(s_1,...,s_k)&=\frac1N \sum_{i=1}^N \prod_{j=1}^k X_i(s_j) \int_{\I^l} u(\tilde{s}_1,...,\tilde{s}_l) \prod_{m=1}^l X_i(\tilde{s}_m) d\mu(\tilde{s_m}), \nonumber \\
k,l&=1,...,p. \label{eq:def_a*a_empirical}
\end{align}
and $[\A^* Y]_N=([A_0^* Y]_N,...,[A_p^* Y]_N) \in \CalL^2$, so that 
\begin{align}
[A_0^* Y]_N&=\frac1N \sum_{i=1}^N Y_i,
\nonumber \\
[A_l^* Y]_N(s_1,...,s_l)&= \frac1N \sum_{i=1}^N Y_i \prod_{j=1}^l X_i(s_j), \nonumber \\
l&=1,...,p. \label{eq:def:a*y_emp}
\end{align}
 
A thorough analysis of one-parameter regularized PFR has been executed in \cite{holzleitner2023regularized} for a generalized regularization scheme, see e.g. Theorem 1 for their main finding. 

Yet, when considering the single-parameter regularization within the realm of PFR, it could be contended that this might not be the most suitable selection. This approach overlooks individual contributions associated with monomials of varying degrees, treating them all with equal weight. In this context, a more fitting alternative for PFR is the employment of MP regularization, a choice that we will discuss thoroughly in Section \ref{sec:multi}. But before that let us move on by collecting several auxiliary results, which have mostly been derived already in \cite{holzleitner2023regularized}.

\subsection{Operator norms and related auxiliary estimates}
Here we collect several estimates related to the norms of the previously discussed operators. 
Most of these results have been discussed in \cite{holzleitner2023regularized}.

\begin{lemma}[Lemma 1 in \cite{holzleitner2023regularized}] \label{lem:hs_bounds}
Let $\text{HS}(\Hs_1, \Hs_2)$ denote the Hilbert space of Hilbert-Schmidt operators between Hilbert spaces $\Hs_1$ and $\Hs_2$. For simplicity let us also use $\text{HS}(\Hs_1, \Hs_1)=\text{HS}(\Hs_1).$
Under Assumption \ref{ass:unif} 
we have that
\begin{align*} 
\norm{\A}_{\text{HS}(\CalL^2, L^2(\Omega, \mathbb{P}))} 
&\le  \tilde{\kappa}=:\sum_{l=0}^p \kappa^l  
\\
\norm{\A^* \A}_{\text{HS}(\CalL^2)}, \norm{[\A^* \A]_N}_{\text{HS}(\CalL^2)} 
&\le  \tilde{\kappa}^2 
\end{align*}
\end{lemma}

\begin{lemma}[Lemma 2 in \cite{holzleitner2023regularized}] \label{lem:op_est_0}
For any $\delta \in (0,1)$, with confidence at least $1-\delta$ we have that

\begin{align} \label{eq:op_est_0}
\left\| \A^* \A -[\A^* \A]_N  \right\|_{\CalL^2 \to \CalL^2} \le \left\| \A^* \A -[\A^* \A]_N  \right\|_{\text{HS}(\CalL^2)} \leq \frac{4 \tilde{\kappa}^2}{\sqrt{N}} \log \frac{2}{\delta}
 \end{align}
 \end{lemma}

 \begin{lemma}[compare with Lemma 4 in \cite{holzleitner2023regularized}] \label{lem:op_est_2_3}
 For any $\delta \in (0,1)$, with confidence at least $1-\delta$ we have that in case of noise assumption \eqref{eq:noise_1}:
\begin{align} \label{eq:op_est_2}
    \left\|[\A^* \A]_N u^+ -[\A^* Y]_N \right\|_{\CalL^2} \leq  \frac{\sigma \tilde{\kappa}}{\sqrt{N} \delta},
\end{align}
whereas in case of \eqref{eq:noise_2}:
\begin{align} \label{eq:op_est_3}
    \left\|[\A^* \A]_N u^+ -[\A^* Y]_N \right\|_{\CalL^2} \leq  \frac{(M+\sigma) \tilde{\kappa} \log (2 / \delta)}{\sqrt{N} }.
\end{align}
\end{lemma}
Lemma \ref{lem:op_est_2_3} can, to some extend, be seen as a special case of Lemma 4 in \cite{holzleitner2023regularized}, however, in order to introduce notation and techniques required for some further technical results as e.g. Lemma \ref{lem:g_and_G_est}, we still decided to provide its proof here:

To this end we also need to recall the following well-known concentration bound:
\begin{lemma}[see e.g. Theorem 3.3.4. in \cite{yurinsky1995sums}] 
\label{lem:concentration}
 Let $\xi$ be a random variable with values in a Hilbert space $\Hs$. Let $\left\{\xi_1, \xi_2, \ldots, \xi_N\right\}$ be a sample of $N$ independent observations of $\xi$.
Furthermore, assume that the bound $\EXP\|\xi\|_{\Hs}^{\tilde{m}} \leqslant \frac{v}{2} \tilde{m} ! u^{\tilde{m}-2}$ holds for every $2 \leqslant \tilde{m} \in \mathbb{N}$, then for any $0<\delta<1$ with confidence at least $1-\delta$ we have
\begin{align*}
\left\|\frac{1}{N} \sum_{i=1}^N\left[\xi_i-\mathbb{E}(\xi)\right]\right\|_{\Hs} \leqslant \frac{2 u \log (2 / \delta)}{N}+\sqrt{\frac{2 v \log (2 / \delta)}{N}}
\end{align*}

\end{lemma}


\begin{proof} [Proof of Lemma \ref{lem:op_est_2_3}]
Let us first focus on the more involved estimate \eqref{eq:op_est_3}. The estimate \eqref{eq:op_est_2} can be proven by similar reasoning. Consider the matrix of operators
\begin{align*} 
\mathcal{A}(\omega)&=\left\{ \mathcal{A}_{k,l}(\omega):  L^2_l \to L^2(\Omega, \mathbb{P}) \otimes L^2_k, k,l=0,...,p \right\},
\end{align*}
where $\mathcal{A}_{0,0}u(\omega) =u,\;\mathcal{A}_{0,l}u (\omega) = (A_l u)(\omega)$,
\begin{align}
\mathcal{A}_{k,l}u(\omega,s_1,...,s_k)
&=\prod_{j=1}^k X(\omega, s_j) \int_{\I^l} u(\tilde{s}_1,...,\tilde{s}_l) \prod_{m=1}^l X(\omega, \tilde{s}_m) d \tilde{s}_m  \label{eq:def:calA_entries}\\
k,l&=1,...,p, \omega \in \Omega \nonumber.
\end{align} 
Then the operators $ \mathcal{A}^i$, $i=1,...,N$, defined by using  $X_i(\cdot)$ instead of $X(\omega, \cdot)$ in the above formulas, can be seen as independent observations of $\mathcal{A}(\omega)$.
 
It is clear that $\EXP(\mathcal{A}(\omega))=\A^* \A$ 
and that $\norm{\mathcal{A}(\omega)}_{\text{HS}(\CalL^2)} \le \tilde{\kappa}^2$, so that $\mathcal{A}(\omega)$ is a random variable in $\text{HS}(\CalL^2)$. Moreover we introduce the vectors
$\mathcal{X}(\omega) \in L^2(\Omega,\mathbb{P} ) \otimes \CalL^2$,
\begin{align} \label{eq:def:calX}
\mathcal{X}(\omega)&= (\mathcal{X}_k(\omega))^p_{k=0}, \; \;\mathcal{X}_0(\omega)=1, \; \; \mathcal{X}_k(\omega)= \prod_{j=1}^k X(\omega,s_j), \; \; k=1,...,p,
\end{align}
with $\norm{\mathcal{X}(\omega)}_{\CalL^2} \le \tilde{\kappa}$, and the $\CalL^2$-valued random variable 
\begin{align*}
\xi(\omega)&=
(Y(\omega) - \A(\omega)u^+) \mathcal{X}(\omega)
=\varepsilon(\omega)\mathcal{X}(\omega),
\end{align*}
where the last equality is due to Assumption \ref{as:noise_model}. 
Then the functions 
\begin{align*}
\xi_i= Y_i\X^i -\mathcal{A}^i u^+,
\end{align*}
where $\mathcal{X}^i$ are defined by using  $X_i(\cdot)$ instead of $X(\omega, \cdot)$ in \eqref{eq:def:calX}, can be seen as independent observations of $\xi(\omega)$. Moreover we have:
\begin{align*}
\frac1N \sum_{i=1}^N \xi_i= \frac1N \sum_{i=1}^N Y_i \X^i -\frac1N \sum_{i=1}^N \mathcal{A}^i u^+ ,
\end{align*}
so that for $k=0,...,p$, recalling \eqref{eq:def:a*y_emp}:
\begin{align} \label{eq:ay_relation}
\left(\frac1N \sum_{i=1}^N Y_i \X^i\right)_k(s_1,...,s_k)= \frac1N \sum_{i=1}^N  Y_i \prod_{j=1}^k X_i(s_j)=[A_k^* Y]_N(s_1,...,s_k),
\end{align}
and recalling \eqref{eq:def_a*a_empirical}:
\begin{align}
&\left(\frac1N \sum_{i=1}^N \mathcal{A}^i u^+ \right)_k (s_1,...,s_k) \nonumber \\
&=\frac1N \sum_{i=1}^N  \sum_{l=0}^p \prod_{j=1}^k X_i(s_j) \int_{\I^l} u_l^+(\tilde{s}_1,...,\tilde{s}_l) \prod_{m=1}^l X_i(\tilde{s}_m) d\mu(\tilde{s}_m) \nonumber \\
&=([\A^* \A]_N u^+)_k(s_1,...,s_k), \label{eq:A*A_N_cal_A_relation}
\end{align}
which allows us to conclude:
\begin{align*}
\frac1N \sum_{i=1}^N \xi_i=[\A^* Y]_N-[\A^* \A]_N u^+.
\end{align*}

Due to Assumption \ref{as:noise_model} we have
\begin{align*}
\EXP(\xi(\omega))= \EXP(\mathcal{X}(\omega))\EXP(\varepsilon(\omega))=0.
\end{align*}

Moreover, the independence of $\varepsilon$ and $\mathcal{X}$ leads to the conclusion that: 
\begin{align*}
\EXP(\norm{\xi}^{\tilde{m}}_{\CalL^2}) &\le  \EXP(\norm{\mathcal{X}(\omega)}_{ \CalL^2}^{\tilde{m}})
\cdot \EXP(|\varepsilon(\omega)|^{\tilde{m}}) 
\le \frac{\sigma^2 \tilde{\kappa}^2}{2} \left( M\tilde{\kappa} \right)^{\tilde{m}-2} \tilde{m}!.
\end{align*}
Now the application of  Lemma \ref{lem:concentration} for $\xi(\omega)$ and $\xi_i$ yields the desired bound \eqref{eq:op_est_3}. 

To obtain \eqref{eq:op_est_2} we need to follow the same lines of proof, but only consider the case $\tilde{m}=2$ and afterward apply Tschebyshev's inequality instead.

\end{proof}

\section{Multiparameter regularization} \label{sec:multi}
Equipped with the necessary background and notation on PFR, let us continue our discussion on MP regularization in this context.
Instead of dealing with \eqref{eq:tikpftfr} and using only a single parameter $\lambda$, we consider a vector $\blam=(\lambda_0,...\lambda_p)$ of the regularization parameters $\lambda_l \geq 0$, $l=0,...,p$, and the corresponding regularization functional with multiple penalties:
\begin{align} \label{eq:pftfr_multi_reg}
\norm{Y-\A u}_{L^2(\Omega, \mathbb{P})}^2 + \sum_{l=0}^p \lambda_l \norm{u_l}_{L_l^2}^2 \to \text{min}.
\end{align} 

Next, let $P_l: \CalL^2 \to \CalL^2$, $l=0,...,p$,  be the projection of $\CalL^2$ onto $L^2_l$, i.e. $P_l(u_0,...,u_p)=(0,..,0,u_l,0,...0)$, so that clearly $\norm{P_l u}_{\CalL^2} = \norm{u_l}_{L_l^2}^2$. Then \eqref{eq:pftfr_multi_reg} can equivalently be written as 
\begin{align*} 
\norm{Y-\A u}^2_{L^2(\Omega, \mathbb{P})} + \sum_{l=0}^p \lambda_l \norm{P_l u}_{\CalL^2}^2 \to \text{min},
\end{align*} 
and using arguments similar to those given in \cite[Formula (3.12)]{lu2013regularization}), we can easily make the following observation:

\begin{lemma}
The minimizer $u_{\blam}$ of \eqref{eq:pftfr_multi_reg} solves the equation
\begin{align} \label{eq:tikpftfr_multi}
\sum_{l=0}^p \lambda_l P_l u &+ \A^* \A u= \A^* Y.
\end{align}
 \end{lemma}

\begin{proof}

The Frechét derivative of $F(u)=\norm{Y-\A u}^2_{L^2(\Omega, \mathbb{P})}$ in direction $v \in \CalL^2$
is given by $F'(u)(v)=\innerpro{2 ( \A^* \A u-\A^* Y),v}_{\CalL^2}$, and that of $F_l(u)=\lambda_l \norm{P_l u}_{\CalL^2}^2$ by $F'_l(u)(v)=2 \lambda_l \innerpro{P_l^* P_l u,v}_{\CalL^2}=2 \lambda_l \innerpro{P_l u,v}_{\CalL^2}$. Setting the derivative of $F+ \sum_{l=0}^p F_l$ to zero and using the convexity of the problem under consideration we arrive at \eqref{eq:tikpftfr_multi}.

\end{proof}

Next we employ a Monte-Carlo type discretization of \eqref{eq:tikpftfr_multi} and approximate the minimizer of \eqref{eq:pftfr_multi_reg} by the solution $u_{\blam}^N$ of 
\begin{align} \label{eq:tikpftfremp_multi}
\sum_{l=0}^p \lambda_l P_l u &+ [\A^* \A]_N u= [\A^* Y]_N 
\end{align}

In view of \eqref{eq:def_a*a_empirical}, \eqref{eq:def:a*y_emp} the regularized approximation $u_{\blam}^N$ 
can be constructed in the form $u_{\blam}^N=(u_{\blam,0}^N,u_{\blam,1}^N,...,u_{\blam,p}^N) \in \CalL^2$, so that
\begin{align*}
u_{\blam,0}^N&=b_0 \in \Real,\\
u_{\blam,l}^N(s_1,...,s_l)&=\sum_{i=1}^N b_{l,i} \prod_{j=1}^l X_i(s_j) \in L^2_l, \quad l=1,...,p.
\end{align*}

Inserting this ansatz into \eqref{eq:tikpftfremp_multi} and equating the corresponding coefficients 
we obtain the following system of $pN+1$ linear equations for $b_0$ and $b_{k,i}$, $k=1,...,p$, $i=1,...,N$:
\begin{align} 
(\lambda_0+1)b_0+\frac1N \sum_{i=1}^N \sum_{l=1}^p \sum_{s=1}^N b_{l,s} (c_{i,s})^l &= \frac1N \sum_{i=1}^N Y_i,\label{eq:b_0_coeff} 
\end{align}
\begin{align} 
&\lambda_k b_{k,i}+\frac1N b_0+\frac1N \sum_{l=1}^p \sum_{s=1}^N b_{l,s} (c_{i,s})^l = \frac1N  Y_i, \label{eq:b_ki_coeff}
\end{align}
where $c_{i,s}=\int_{\mathbb{I}} X_i(\tilde{s}) X_s(\tilde{s}) d\mu(\tilde{s})$. 
Note that for single-parameter regularization, i.e. for the case of $\lambda=\lambda_0=...=\lambda_p$, the system \eqref{eq:b_0_coeff}, \eqref{eq:b_ki_coeff}
allows for a reduction to a linear system of only $N+1$ equations. This was discussed in detail in Section 3 of \cite{holzleitner2023regularized}.


A crucial issue, however, 
is the choice of the regularization parameters $\lambda_0,...,\lambda_p$. In Section \ref{sec:intro} we already mentioned several approaches to this issue. But most of them select just one set of parameters. On the other hand, it seems more practical to use all values from a grid of parameters
and then aggregate all the resulting models, such that even badly chosen regularization parameters can in the end contribute to
an improved model. In the next section we will theoretically justify an aggregation method in the context of PFR and also observe its usefulness in the empirical evaluations in Sections \ref{sec:empirical_toy}--\ref{sec:empirical_stenosis}. 

\section{Aggregation of multiple regularized polynomial functional models} \label{sec:agg}

To continue with a discussion of an aggregation strategy in the PFR context, let us now assume that we are given a sequence of models $u_1,...,u_R \in \CalL^2$, so that the following assumption is valid:
\begin{assumption}
    \label{ass:unif_bound_models}
    \begin{align*}
    \norm{u^+}_{\CalL^2},  \norm{u_r}_{\CalL^2} \le C_R
    \end{align*}
    for $r=1,\cdots,R$ and some $C_R >0$.
\end{assumption}

Let us observe that this setting encompasses the aggregation of any $R$ models in $\CalL^2$, and is therefore directly applicable to aggregated models obtained via MP regularization.

Our goal is to compute an aggregation 
\begin{align} \label{eq:aggdef}
\sum_{r=1}^R c_{r} u_r
\end{align}
with coefficients $c_1,...,c_r\in \Real$, so that the excess of risk
$\mathcal{E}(\sum_{r=1}^R c_{r} u_r)-\mathcal{E}(u^+)$ is as small as possible. We already know from 
\eqref{eq:norm_comparisons1} and \eqref{eq:noise_model} that 
\begin{align} \label{eq:error_decomp}
\mathcal{E}(u)-\mathcal{E}(u^+)=\norm{\A \left(u - u^+ \right)}_{L^2(\Omega, \mathbb{P}) }^2=\norm{\sqrt{\A^* \A} \left(u - u^+ \right)}_{\CalL^2}^2,
\end{align}
for any $u \in \CalL^2$, so that our main objective can be written as follows:
\begin{align} \label{eq:agg_obj}
    \min_{c_1,...,c_R\in \Real} \norm{\sqrt{\A^* \A} \left( \sum_{r=1}^R c_{r} u_{r} - u^+ \right)}_{\CalL^2}^2.
\end{align}
Next we observe that the minimizer of \eqref{eq:agg_obj}, i.e. the best approximation $u^*$ of the target regression function $u^+$ by linear combinations, corresponds to the vector $c^*=(c_1^*,\ldots,c_R^*)$ of ideal coefficients in \eqref{eq:aggdef} that solves the linear system $G c^* = \bar{g}$ with the Gram matrix 
\begin{align*}
G = \left(  \innerpro{ \sqrt{\A^* \A} u_r , \sqrt{\A^* \A} u_{r'} }_{\CalL^2} \right)_{r,r'=1}^R=\left(  \innerpro{ \A u_r , \A u_{r'} }_{L^2(\Omega, \mathbb{P})} \right)_{r,r'=1}^R
\end{align*}
and the vector 
\begin{align*}
\bar{g} = \left(\innerpro{\sqrt{\A^* \A} u^+ , \sqrt{\A^* \A} u_r }_{\CalL^2} \right)_{r=1}^R=\left(\innerpro{\A u^+ , \A u_r }_{L^2(\Omega, \mathbb{P})} \right)_{r=1}^R
\end{align*}
(see e.g. \cite[Section 3.5.]{pereverzyev2022introduction} for a proof of this well known observation).


Note that the successful inversion of $G$ depends on the assumption that our models exhibit sufficient dissimilarity. This requirement is inherent, as without it, we could effortlessly eliminate redundant models. But, of course, neither Gram matrix $G$ nor the vector $\bar{g}$ is accessible, because there is no access to $\mathbb{P}$, so we switch to the empirical counterparts $\tG$ and $\tg$, i.e. 
\begin{align}
\tG &=\left( \frac1N \sum_{i=1}^N (\A_i u_r)( \A_i u_{r'}) \right)_{r,r'=1}^R  \label{eq:tildeG_def}\\
    \tg &= \left( \frac1N \sum_{i=1}^N Y_i (\A_i u_r)   \right)_{r=1}^R. \label{eq:tildeg_def}
    \end{align}
Then we compute the solution $\tilde{c}=(\tilde{c}_1,...,\tilde{c}_R)$ to the system $\tG \tilde{c} =\tg$, so that our aggregated model is given by 
\begin{align} \label{eq:utilde_def}
\tilde{u}=\sum_{r=1}^R \tilde{c}_{r} u_r.
\end{align}
Our main result is about the quality of this aggregation computed from data, and we show that $\mathcal{E}(\tilde{u})- \mathcal{E}(u^+)$ approaches $2(\mathcal{E}(u^*)- \mathcal{E}(u^+))$ when the  sample size increases:
\begin{theorem}
\label{thm:main}
Under assumptions \ref{ass:unif} -- \ref{ass:unif_bound_models} with probability $1-\delta$ it holds that for sufficiently large $N$
\begin{align}
&\mathcal{E}(\tilde{u})- \mathcal{E}(u^+)
\leq 2  \left(\mathcal{E}(u^*)- \mathcal{E}(u^+) \right)
+C  N^{-1} \log^2 \frac{1}{\delta}, 
\label{eq:main_gen_bound}
\end{align}
where the coefficient $C>0$ does not depend on $N$ and $\delta$.
\end{theorem}

According to our theorem, the excess of risk of the proposed algorithm is asymptotically not worse than
twice the excess of risk of the unknown optimal aggregation, because it is clear (see, e.g., Corollary 1 in \cite{holzleitner2023regularized}) that the second term in the right hand side of \eqref{eq:main_gen_bound} is negligibly small.

The proof of this result will crucially depend on the following Lemma, which relates the entries of $G$ and $\tilde{G}$ and $\bar{g}$ and $\tilde{g}$, respectively:

\begin{lemma} \label{lem:g_and_G_est}
Under assumptions \ref{ass:unif} -- \ref{ass:unif_bound_models} with probability $1-\delta$ we have that for any $r,r'=1,...,R$ :
\begin{align} \label{eq:G_estimate}
\abs{\innerpro{\A u_r, \A u_{r'}}_{L^2(\Omega, \mathbb{P})}-\frac1N \sum_{i=1}^N \A_i u_r \A_i u_{r'}}  \le  \frac{C \log \frac{2}{\delta}}{ \sqrt{N}}, \\
\abs{\innerpro{\A u^+, \A u_r}_{L^2(\Omega, \mathbb{P})}-\frac1N \sum_{i=1}^N Y_i \A_i u_r} \le \frac{C \log \frac{2}{\delta}}{ \sqrt{N}}, \label{eq:g_estimate}
\end{align}
where $C$ is some generic constant $C$, which does not depend on $N$ or $\delta$.
\end{lemma}

\begin{proof}
Let us start by showing \eqref{eq:G_estimate}. 
Observe that in view of \eqref{eq:def_a*a_empirical} we have
\begin{align*}
\frac1N \sum_{i=1}^N \A_i u_r \A_i u_{r'}&=\frac1N \sum_{i=1}^N \innerpro{\X_i, u_r}_{\CalL^2} \innerpro{\X_i, u_{r'}}_{\CalL^2}
= \innerpro{[\A^* \A]_N u_r,u_{r'}}_{\CalL^2},
\end{align*}
Then
\begin{align*}
    \innerpro{\A^* \A u_r, u_{r'}}_{\CalL^2}=\innerpro{[\A^* \A]_N u_r, u_{r'}}_{\CalL^2}+\innerpro{(\A^* \A-[\A^* \A]_N) u_r, u_{r'}}_{\CalL^2},
\end{align*}
and it remains to estimate the last term to arrive at \eqref{eq:G_estimate}:
\begin{align*}
\abs{\innerpro{(\A^* \A-[\A^* \A]_N) u_r, u_{r'}}_{\CalL^2}} &\le\norm{ \A^* \A -[\A^* \A]_N  }_{\CalL^2 \to \CalL^2} \norm{u_r}_{\CalL^2} \norm{u_{r'}}_{\CalL^2}  \\ &\le C_R^2 \frac{4 \tilde{\kappa}^2}{\sqrt{N}} \log \frac{2}{\delta},
\end{align*}
where we used Cauchy-Schwartz inequality, Lemma \ref{lem:op_est_0} and Assumption \ref{ass:unif_bound_models}.

Now let us deal with \eqref{eq:g_estimate}. 
It is clear from \eqref{eq:ay_relation} that
\begin{align*}
    \innerpro{[\A^* Y]_N,u_r}_{\CalL^2} = \frac1N \sum_{i=1}^N Y_i \innerpro{\X_i, u_r}_{\CalL^2}=\frac1N \sum_{i=1}^N Y_i \A_i u_r.
\end{align*}
Then we can continue as follows:
\begin{align*}
   \innerpro{\A^* \A u^+, u_r}_{\CalL^2}=\innerpro{[\A^* Y]_N,u_r}_{\CalL^2}&+\underbrace{\innerpro{(\A^* \A-[\A^* \A]_N) u^+, u_r}_{\CalL^2}}_{(I)}\\ &+\underbrace{\innerpro{[\A^* \A]_N u^+ -[\A^* Y]_N, u_r}_{\CalL^2}}_{(II)}.
\end{align*}
For (I) we apply Lemma \ref{lem:op_est_0}, Assumption \ref{ass:unif_bound_models} and Cauchy-Schwartz inequality to obtain the bound:
\begin{align*}
(I) \le C_R^2 \frac{4 \tilde{\kappa}^2}{\sqrt{N}} \log \frac{2}{\delta},
\end{align*}
whereas for (II) we use \eqref{eq:op_est_2} or \eqref{eq:op_est_3} and again Assumption \ref{ass:unif_bound_models} and Cauchy-Schwartz to have:
\begin{align*}
(II) \le C_R \frac{C}{\sqrt{N} }\log \frac{2}{\delta},
\end{align*}
where the constant $C$ may be different, depending on whether noise assumption \eqref{eq:noise_1} or \eqref{eq:noise_2} is in force.
Now \eqref{eq:g_estimate} follows by combining (I) and (II). 
\end{proof}

Now we can use similar arguments as used, e.g., in the proof of Theorem 1 of \cite{dinu2023addressing}. In the sequel, $\norm{.}_{\Real^R}$ and $\norm{.}_{\Real^R \to \Real^R}$ denote the usual Euclidean and the Frobenius norm, respectively. From Lemma \ref{lem:g_and_G_est} we can argue that with probability $1-\delta$ it holds:
\begin{align}
&\norm{\bar{g}-\tilde{g}}_{\Real^R }\le  C \log \frac{1}{\delta} N^{-\frac{1}{2}}  \label{eq:g_diff_bound}, \\
&\norm{G-\tilde{G}}_{\Real^R \to \Real^R}\le C \log \frac{1}{\delta}  N^{-\frac{1}{2}} \label{eq:G_diff_bound}.
\end{align}

It is also straightforward to bound the entries of $\tilde{G}$ uniformly:
\begin{align*}
    |\tilde{G}_{r,r'}| 
    \le \norm{ [\A^* \A]_N  }_{\CalL^2 \to \CalL^2} \norm{u_r}_{\CalL^2} \norm{u_{r'}}_{\CalL^2} \le  \tilde{\kappa}^2 C_R^2.
\end{align*}
Moreover we can use the following simple manipulation:
\begin{align*}
    G^{-1}=\tilde{G}^{-1}(G\tilde{G}^{-1})^{-1}=\tilde{G}^{-1}(I-(I-G\tilde{G}^{-1}))^{-1}=\tilde{G}^{-1}(I-(\tilde{G}-G)\tilde{G}^{-1})^{-1}.
\end{align*}
Then using the Neumann series for $(I-(\tilde{G}-G)\tilde{G}^{-1})^{-1}$ we obtain the following bound:
\begin{align}
\left\|G^{-1}\right\|_{\Real^R \to \Real^R} \leq \frac{\left\|\tilde{G}^{-1}\right\|_{\Real^R \to \Real^R}}{1-\left\|\tilde{G}^{-1}\right\|_{\Real^R \to \Real^R}\left\|G-\tilde{G} \right\| _{\Real^R \to \Real^R}} \le c.
\label{eq:G_inv_bound}
\end{align}
To see \eqref{eq:G_inv_bound}, we first observe that it is natural to assume that $\norm{\tilde{G}^{-1}} \le c$ for some generic $c>0$ (otherwise, we can, e.g., orthogonalize our models and coefficients without changing the aggregation, but with reducing the condition number $\norm{\tilde{G}^{-1}}\norm{\tilde{G}}$). Secondly, by \eqref{eq:G_diff_bound} it is also natural to assume that $\norm{G-\tilde{G}}_{\Real^R \to \Real^R}  < \frac{1}{2c}$ by choosing $N$ sufficiently large. Therefore the Neumann series associated to $(I-(\tilde{G}-G)\tilde{G}^{-1})^{-1}$ converges, since $\norm{(\tilde{G}-G)\tilde{G}^{-1}}_{\Real^R \to \Real^R} < \frac12$. This allows to deduce $\left\|G^{-1}\right\|_{\Real^R \to \Real^R} \le 2c$.
Now we are in the position to prove our main generalization bound \eqref{eq:main_gen_bound}:

\begin{proof}[Proof of Theorem \ref{thm:main}]
Since:
\begin{align*}
    G^{-1}(\tilde{g}-\bar{g})+G^{-1}(G-\tilde{G})\tilde{c}=G^{-1} \tilde{g}-c^*+\tilde{c}-G^{-1}\tilde{g}=\tilde{c}-c^*,
\end{align*}
we can use \eqref{eq:g_diff_bound}--\eqref{eq:G_inv_bound} and Hölder’s inequality 
to claim that for sufficiently large $N$ with probability $1-\delta$ it holds

\begin{align}
\|\tilde{c}-c^*\|^2_{\mathbb{R}^{R}} & \leq 2 \left\|G^{-1}\right\|_{\Real^R \to \Real^R}^2 \left(\|\tilde{g}-\bar{g}\|_{\mathbb{R}^{R}}^2+\|G-\tilde{G}\|_{\Real^R \to \Real^R}^2 \|\tilde{c}\|_{\mathbb{R}^{R}}^2\right) \nonumber \\ &\le C N^{-1} \log^2 \frac{1}{\delta}   \label{eq:c_diff}
\end{align}
Moreover:
\begin{align}
\mathcal{E}(\tilde{u}) - \mathcal{E}(u^+)&=\norm{\sqrt{\A^* \A} (\tilde{u} - u^+)}_{\CalL^2}^2 \nonumber \\
&\le  \left( \norm{\sqrt{\A^* \A} (u^*-u^+)}_{\CalL^2} + \norm{\sqrt{\A^* \A} (\tilde{u} - u^*)}_{\CalL^2} \right)^2 \nonumber \\
&\le  2 \norm{\sqrt{\A^* \A} (u^*-u^+)}_{\CalL^2}^2 + 2 \norm{\sqrt{\A^* \A} (\tilde{u} - u^*)}_{\CalL^2}^2 \nonumber \\
&=  2    \left(\mathcal{E} \left(u^* \right) - \mathcal{E} (u^+) \right) + 2 \norm{\sqrt{\A^* \A} (\tilde{u} - u^*)}_{\CalL^2}^2 \nonumber \\
&\le 2    \left(\mathcal{E} \left(u^* \right) - \mathcal{E} (u^+) \right) + 2 \left( \sum_{r=1}^R |c^*_k-\tilde{c}_r| \norm{\sqrt{\A^* \A} u_r}_{\CalL^2} \right)^2  \nonumber\\
& \leq2    \left(\mathcal{E} \left(u^* \right) - \mathcal{E} (u^+) \right)+2 R\|c^*-\tilde{c}\|_{\mathbb{R}^R}^2 \max _{r}\left\| \sqrt{\A^* \A} u_r\right\|_{{\CalL^2}}^2 \nonumber \\
& \leq2    \left(\mathcal{E} \left(u^* \right) - \mathcal{E} (u^+) \right)+2 R \norm{\sqrt{\A^* \A}}_{\CalL^2 \to \CalL^2}^2 C_R^2 \|c^*-\tilde{c}\|_{\mathbb{R}^{R}}^2 \label{eq:final_bound},
\end{align}
The statement of the theorem follows now from \eqref{eq:c_diff}--\eqref{eq:final_bound}.
\end{proof}

Some remarks on the interpretation of the main theorem are in order:

\begin{remark}
From a theoretical perspective, we observe that the optimal aggregation $u^*$ of several models consistently yields improved results compared to relying on any single model. Indeed, we have:
\begin{align*}
\min_{c_1,\dots,c_R\in \mathbb{R}} \left\| \sqrt{\A^* \A} \left( \sum_{r=1}^R c_{r} u_{r} \right) \right\|_{\CalL^2}^2 
&\le \min_{\substack{c_i = 1,\; c_j = 0 \\ i = 1, \cdots, R,\; j = 1, \cdots, i-1, i+1, \cdots, R}} \left\| \sqrt{\A^* \A} \left( \sum_{r=1}^R c_{r} u_{r} \right) \right\|_{\CalL^2}^2 \\
&= \min_{r=1,\dots, R} \left\| \sqrt{\A^* \A} u_r \right\|_{\CalL^2}^2.
\end{align*}
Since MP regularization allows for a more flexible model class than single-parameter regularization, and since we compute an approximate aggregation $\tilde{u}$ which provably converges to $u^*$ as $N \to \infty$, it is natural to expect that our proposed aggregation approach, which combines models with varying parameter choices, will outperform the method proposed in \cite{holzleitner2023regularized}. This expected improvement is also confirmed by our experimental results.
\end{remark}

\begin{remark}
The choice of $R$ remains an important aspect, which is primarily determined by the specific task and practical constraints. In particular, $R$ is often limited by the available computational resources, as computing the individual models may be expensive. In practice, $R$ typically remains moderate, commonly ranging between $10$ and $50$ (see, e.g., \cite{dinu2023addressing} as well as our experiments in the subsequent sections). From a theoretical standpoint, we examine the influence of $R$ on the error bounds established in Theorem \ref{thm:main} in the following corollary.
\end{remark}

\begin{corollary}
\label{cor:main}
Under the same assumptions as in Theorem \ref{thm:main} with probability $1-\delta$ it holds that for sufficiently large $N$
\begin{align}
&\mathcal{E}(\tilde{u})- \mathcal{E}(u^+)
\leq 2  \left(\mathcal{E}(u^*)- \mathcal{E}(u^+) \right)
+C  R^3 N^{-1} \log^2 \frac{1}{\delta}, 
\label{eq:main_gen_bound_with_R}
\end{align}
where the coefficient $C>0$ does not depend on $N$, $R$ and $\delta$.
\end{corollary}
\begin{proof}
To analyze the dependence on $R$, we consider the second term in the last line of \eqref{eq:final_bound}. The term $\norm{\sqrt{\A^* \A}}_{\CalL^2 \to \CalL^2}^2$ involves the target operator and is independent of the number of models $R$. Thus, it remains to analyze the term $\|c^*-\tilde{c}\|_{\mathbb{R}^{R}}^2$. For this purpose, we revisit the individual factors appearing in the first inequality of \eqref{eq:c_diff}. Throughout the proof, $C > 0$ denotes a generic absolute constant, independent of $R$, $N$, and $\delta$.

\paragraph{Bound on $\|G-\tilde{G}\|_{\mathcal{L}(\mathbb{R}^R)}$:}  
By Lemma \ref{lem:g_and_G_est}, each entry of the matrix $G - \tilde{G}$ is bounded in absolute value. The proof of this lemma relies only on norm bounds for the corresponding sampling operators and on the uniform bound $C_R$ for all models. Therefore, the constant $C$ is independent of $R$. Using the definition of the Frobenius norm and standard norm inequalities, we obtain
\begin{equation} \label{eq:G_diff_bound_with_l}
\|G-\tilde{G}\|_{\mathcal{L}(\mathbb{R}^R)} \le C R N^{-\frac{1}{2}} \log^{\frac{1}{2}} \frac{1}{\delta}.
\end{equation}

\paragraph{Bound on $\|\tilde{g}-\bar{g}\|_{\mathbb{R}^{R}}$:}  
Applying similar arguments, we obtain
\[
\|\tilde{g}-\bar{g}\|_{\mathbb{R}^{R}} \le C \sqrt{R} N^{-\frac{1}{2}} \log^{\frac{1}{2}} \frac{1}{\delta} .
\]

\paragraph{Bound on $\|G^{-1}\|_{\mathcal{L}(\mathbb{R}^R)}$:}  
The arguments leading to \eqref{eq:G_inv_bound} ensure that the constant involved can be chosen independently of $R$.

\paragraph{Bound on $\|\tilde{c}\|_{\mathbb{R}^{R}}$:}  
Since this quantity is fully determined by the observed data, it can be regarded as independent of $R$.

\medskip

Combining these estimates, we obtain
\[
\|c^*-\tilde{c}\|_{\mathbb{R}^{R}}^2 \le C R^2 N^{-1} \log \frac{1}{\delta},
\]
which establishes the refined bound stated in \eqref{eq:main_gen_bound_with_R}.
\end{proof}

\section{Experimental evaluation}
\subsection{Toy example} \label{sec:empirical_toy}
In this section we a toy example to demonstrate the advantage of MP regularization and aggregation.
To this end, as an explanatory variable, we consider a random process
\begin{align*}
X(\omega, t)=\sum_{k=0}^5 \xi_k(\omega) \cos (k t), t \in[0,2 \pi],
\end{align*}
where $\xi_k(\omega)$ are random variables uniformly distributed on $[-1,1]$.
Consider also the response variable $Y(\omega)$ related to the explanatory variable $X(\omega, t)$  as follows:

\begin{align*}
 Y(\omega)&=u_0^+ +\int_0^{2 \pi} X(\omega, t) u_1^+(t) d \mu(t) +\int_0^{2 \pi} \int_0^{2 \pi} X(\omega, t) X(\omega, \tau) u_2^+(t, \tau) d \mu(t) d\mu( \tau) .
\end{align*}
In our simulations, we use $u^{+}=\left(u_0^{+}, u_1^{+}, u_2^{+}\right)$ with
\begin{align*}
u_0^{+}=2, u_1^{+}=1+4 \cos t+\cos 5 t, u_2^{+}=\cos 3 t+\cos 2 t \cos 2 \tau.
\end{align*}

We simulate $N$ independent samples  $(Y_i, X_i(\cdot))$of $(Y(\omega), X(\omega, t))$ and use them to construct the regularized quadratic approximation
$u_{\blam}^N= (u_{\blam,0}^N,u_{\blam,1}^N,u_{\lambda,2}^N)$ of $u^{+}=\left(u_0^{+}, u_1^{+}, u_2^{+}\right)$ by MP regularization as described in Section \ref{sec:multi}, for 27 different values of $\blam$, so that all possible choices of $\lambda_0, \lambda_1, \lambda_2 \in \left\{ 10^{-5}, 10^{-7}, 10^{-9} \right\}$ are considered.

On Figure \ref{fig1} we plot the error $\norm{u^{+}-u_{\blam}^N}_{\CalL^2}$ against the number of the used samples $N=1,2, \ldots, 40$, for these 27 choices of $\blam$. In several cases (e.g. $\lambda_0=\lambda_1=10^{-9}, \lambda_2=10^{-7}$) it is clearly visible that choosing different values of $\lambda_0, \lambda_1, \lambda_2$ can be advantageous compared to the case of one-parameter regularization $\lambda_0=\lambda_1=\lambda_2$. We also observe that the error curves corresponding to all the computed models saturate at low values (roughly $\sim 3.14$) for  $N \geq 27$.

Next, to see the advantage of combining all the 27 computed models in terms of an aggregation as discussed in Section \ref{sec:agg}, in Figure \ref{fig2} we even observe saturation at $ \sim 3.14$ already at $N \geq 21$. Let us also mention, that we provided an implementation in Pytorch \cite{paszke2019pytorch}. This allows the code to leverage GPU acceleration, enabling fast computation of the involved integrals.
The for this toy-example is available from the following git-repository: \url{https://github.com/markush314/Polynomial-functional-regression}.

These results look promising, therefore in order to underpin the usefulness of our method, we show experiments on real world medical data in the next subsection.

\begin{figure}
\vspace{-3cm}
\centerline{\includegraphics[width=1.2\textwidth]{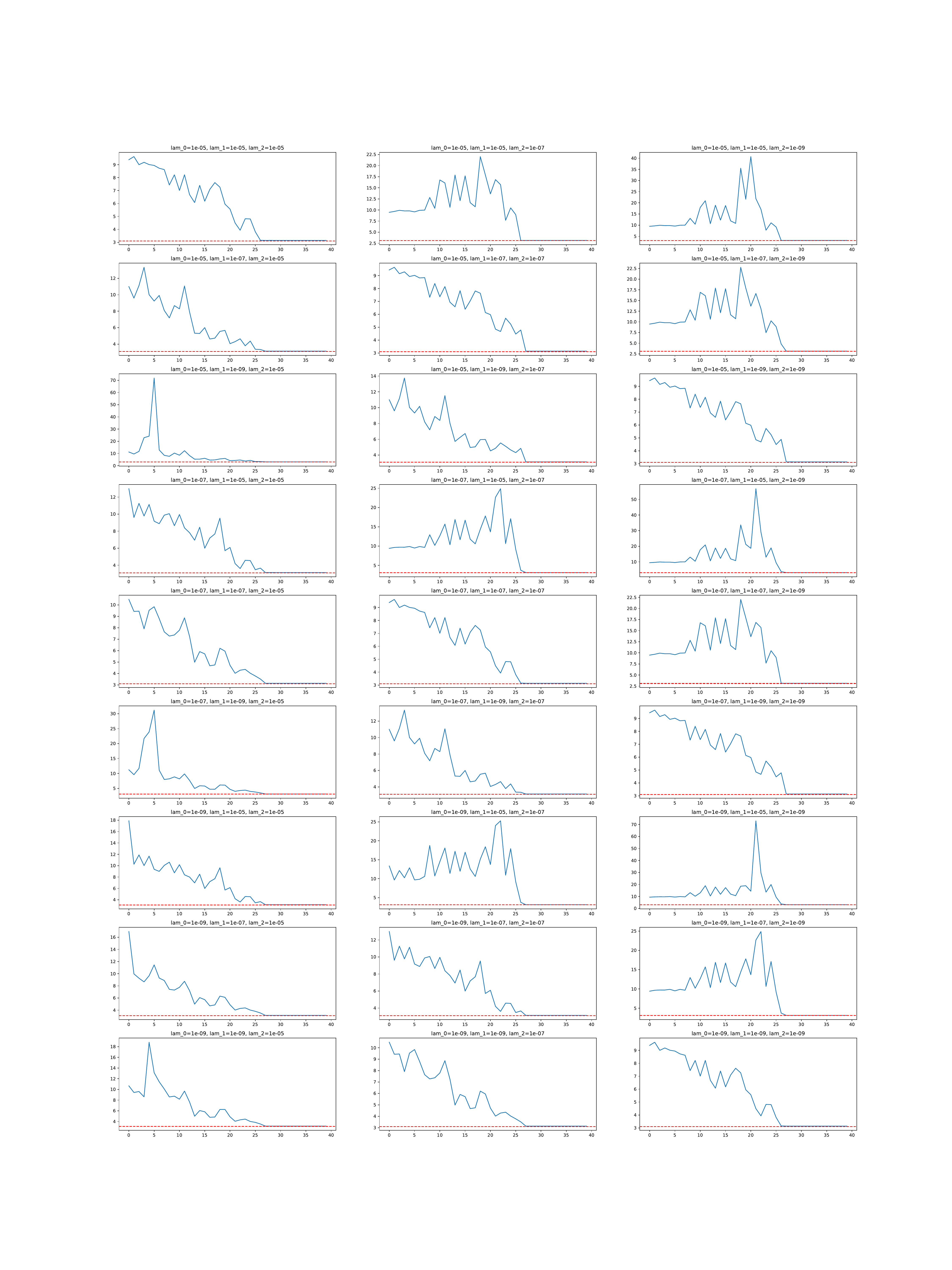}}
\vspace{-1cm}
\caption{Error curves for all possible choices of $\blam$. Red line depicts error rate of $3.14$}
\label{fig1}
\end{figure}

\begin{figure}
\centerline{

\includegraphics[width=0.8\textwidth]{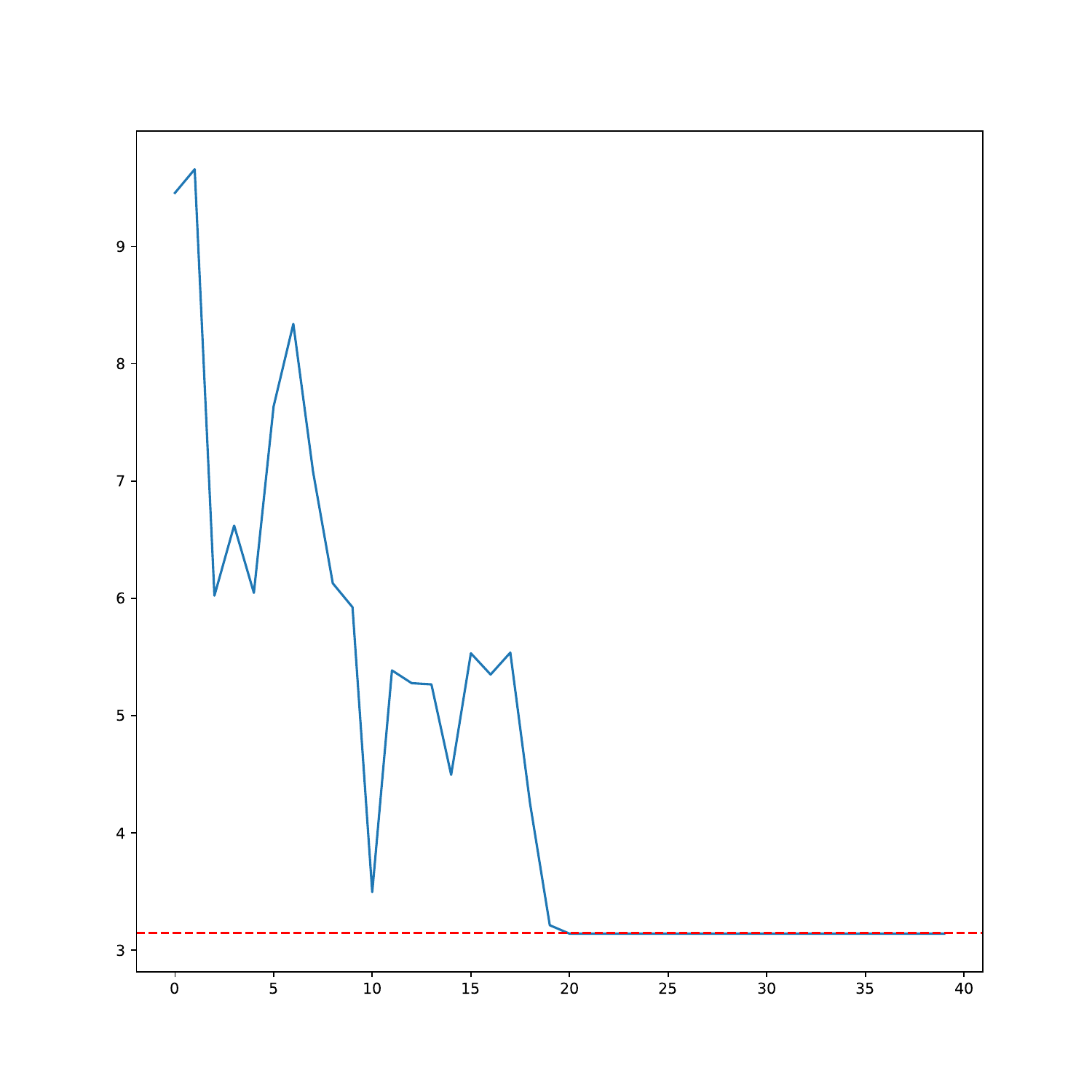}
}
\caption{Error curve for aggregation. Red line depicts error rate of $3.14$}
\label{fig2}
\end{figure}

\subsection{Stenosis Data} \label{sec:empirical_stenosis}

In this section, we demonstrate an application of MP-PFR and the associated aggregation approach presented in Sections \ref{sec:multi}, \ref{sec:agg} to the problem of automatic stenosis detection from lumen diameters. Stenosis refers to an abnormal narrowing of a blood vessel due to a lesion, leading to a reduction in the lumen's space. This pathology is particularly critical in cervical arteries, including internal carotid arteries (ICA) and vertebral arteries (VA), where stenosis can impede or block blood flow to the brain, significantly elevating the risk of a stroke. Consequently, the automatic detection of stenosis becomes a crucial challenge in neuroradiology.

This detection issue typically arises in the final or quantification stage of computerized tomography (CT) or magnetic resonance imaging (MRI) angiography, when the vessel lumen segmentation and centerline extraction have already been executed. The detection mentioned above is the result of all work in the earlier stages, and therefore deserves special attention.
Following the segmentation of CT/MRI scans, the existing software facilitates the estimation of vessel cross-section diameters, denoted as $d_s$ ($s = 1, 2, ...,$ approximately $500$), at various positions $t_s$ along the vessel centerlines. Given the variability in positions $t_s$ and their total numbers across different patients, it is natural to organize these data in the form of functions $X(t)$. For instance, cubic interpolation splines with knots at $t_s$ can describe the diameter variation, with the values $X(t_s)$ corresponding to $d_s$ ($s =1,2,...$). This approach allows clinical data to be represented as a samples $(X_i, Y_i)$ consisting of functional inputs $X_i = X_i(t)$ ($i = 1, 2, ..., N$) labeled by outputs $Y_i$ which are assigned the value of $0$ for a diagnosis indicating no stenosis, and values of $0.25$, $0.5$, $0.75$, $1.0$ for diagnoses representing light, medium, moderate, or high stenosis, respectively. With the use of this dataset, a predictor can be constructed to automatically detect the presence or absence of stenosis by assigning an appropriate label $Y$ to
the corresponding profile $X(t)$ of variations in vessel cross-section diameters.

We have permission for research-driven secondary use of anonymized clinical data collected at the Department of Radiology and Department of Neurology, Medical University of Innsbruck, within the ReSect-study \cite{mayer2019local}. In our experiments below, we use the data about $N=40$ ICA. The available data sample contains only $7$ arteries affected by stenosis, and we need to ensure their inclusion in both the training and test sets. To achieve this, we opt for a random train-test split, so that the training set will consistently comprise of data of $16$ ICA without stenosis and $4$ ICA with stenosis, while the test set will consist of data of $17$ non-stenosis arteries and $3$ stenosis-affected ones.

Recall that in the present context, the variables $X(t)$ are functions of the position $t$ along vessel centerlines, and the lengths of that centerlines vary from patient to patient. Therefore, in order to compute the integrals required in the algorithms of Sections \ref{sec:multi} and \ref{sec:agg}, we confine the inputs to a specific interval $\I=[0,b]$,  where $b$ is the minimum length observed in the available clinical data. In our experiments, we use $b = 140$ mm. Moreover, let us also mention that similar to Section \ref{sec:empirical_toy}, also here we use Pytorch \cite{paszke2019pytorch} for our implementations, so that the code is GPU-compatible and allows for fast computation of the involved integrals.

We construct the models $ u_{\blam}$ as described in \eqref{eq:b_0_coeff}--\eqref{eq:b_ki_coeff}, for both linear and quadratic functional regression, and for all possible choices of  $\lambda_0, \lambda_1, \lambda_2 \in \left\{ 10^{-2}, 10^{-1}, 1 \right\}$. Afterward, we compute an aggregation $\tilde{u}$ of all $u_{\blam}$ corresponding to the different choices of $\blam$, again both for linear and  quadratic case. Hereby we use the approach from Section \ref{sec:agg}, i.e. we solve the system associated to \eqref{eq:tildeG_def}--\eqref{eq:tildeg_def} and then combine the aggregation function \eqref{eq:utilde_def}. For a given functional data sample $X_i$, the predicted value is then computed as $f(X_i)=\A_i u_{\blam}$, or $f(X_i)=\A_i \tilde{u}$, respectively. Note also that when a continuous-valued predictor $f(X_j)$ is used as binary classifier, its diagnostic ability depends on the so-called discrimination threshold $c$, such
that a particular artery corresponding to an input $X_i$ is assumed to be affected by stenosis if $f(X_i)>c$. In our experiments, we choose $c=0.5$.

It is known that in medical statistics the accuracy of prediction of the presence or absence of a medical condition is mathematically described in terms of sensitivity (SE) and specificity (SP). Recall that 
SE is calculated as $\frac{\text{TP}}{\text{TP + FN}}$, while $SP = \frac{\text{TN}}{\text{TN + FP}}$. Here, TP represents the instances where a stenosis in the examined artery is identified by both the reference standard and the algorithm, irrespective of its severity (given the preventive measures for even mild narrowing of the cervical artery). TN accounts for cases where no stenosis in the considered artery is detected by both the reference standard and the algorithm. Meanwhile, FN and FP denote the respective counts of cases where the algorithm incorrectly identifies the absence or presence of stenosis.

The diagnostic efficacy of a specific classifier can also be effectively evaluated using the receiver/relative operating characteristic (ROC) curve. This graphical representation illustrates the diagnostic performance of $f$ across varying discrimination thresholds. The ROC curve is constructed by plotting the sensitivity (SE) against the complement of specificity ($1-\text{SP}$) for different threshold settings. 

The outcomes of ROC analysis can be succinctly summarized using a single metric, namely the area under the ROC curve (AUC). The AUC ranges from approximately $0.5$ for randomly assigned diagnoses to $1.0$, indicating perfect diagnostic classification. In the subsequent analysis, we present the performance of the considered classifiers based on test inputs, considering all the aforementioned metrics and assuming that all classifiers use the same threshold $c=0.5$. 

Our results for the linear and quadratic case are depicted in Tables \ref{tab:linear} and \ref{tab:quad}. We report the performance measure as an average over $10$ runs, both linear, quadratic and aggregated models use the same data for training and testing in each run, so that a fair comparison is provided. 
\begin{table}[h]
\centering
\begin{tabular}{|l|c|c|c|}
\hline
Parameters & SE & SP & AUC \\
\hline
$\lambda_0=1$, $\lambda_1=1$ & 0.333333 & 0.988235 & 0.915686 \\
\hline
$\lambda_0=1$, $\lambda_1=0.1$ & 0.4 & 0.952941 & 0.833333 \\
\hline
$\lambda_0=1$, $\lambda_1=0.01$ & 0.333333 & 0.911765 & 0.668627 \\
\hline
$\lambda_0=0.1$, $\lambda_1=1$ & 0.333333 & 1 & 0.915686 \\
\hline
$\lambda_0=0.1$, $\lambda_1=0.1$ & 0.4 & 0.958824 & 0.833333 \\
\hline
$\lambda_0=0.1$, $\lambda_1=0.01$ & 0.333333 & 0.911765 & 0.668627 \\
\hline
$\lambda_0=0.01$, $\lambda_1=1$ & 0.266667 & 1 & 0.933333 \\
\hline
$\lambda_0=0.01$, $\lambda_1=0.1$ & 0.366667 & 0.964706 & 0.839216 \\
\hline
$\lambda_0=0.01$, $\lambda_1=0.01$ & 0.333333 & 0.923529 & 0.67451 \\
\hline
Aggregation & 0.6 & 0.488235 & 0.641176 \\
\hline
\end{tabular}
\vspace{5pt}
\caption{Performance metrics for linear MP-FR, averaged over $10$ runs.}
\label{tab:linear}
\end{table}

\begin{table}[h]
\centering
\begin{tabular}{|l|c|c|c|}
\hline
Parameters & SE & SP & AUC \\
\hline
$\lambda_0=1$, $\lambda_1=1$, $\lambda_2=1$ & 0.6 & 0.482353 & 0.54902 \\
\hline
$\lambda_0=1$, $\lambda_1=1$, $\lambda_2=0.1$ & 0.6 & 0.488235 & 0.552941 \\
\hline
$\lambda_0=1$, $\lambda_1=1$, $\lambda_2=0.01$ & 0.6 & 0.488235 & 0.552941 \\
\hline
$\lambda_0=1$, $\lambda_1=0.1$, $\lambda_2=1$ & 0.633333 & 0.417647 & 0.492157 \\
\hline
$\lambda_0=1$, $\lambda_1=0.1$, $\lambda_2=0.1$ & 0.6 & 0.482353 & 0.54902 \\
\hline
$\lambda_0=1$, $\lambda_1=0.1$, $\lambda_2=0.01$ & 0.6 & 0.488235 & 0.552941 \\
\hline
$\lambda_0=1$, $\lambda_1=0.01$, $\lambda_2=1$ & 0.766667 & 0.270588 & 0.4 \\
\hline
$\lambda_0=1$, $\lambda_1=0.01$, $\lambda_2=0.1$ & 0.633333 & 0.411765 & 0.490196 \\
\hline
$\lambda_0=1$, $\lambda_1=0.01$, $\lambda_2=0.01$ & 0.6 & 0.482353 & 0.54902 \\
\hline
$\lambda_0=0.1$, $\lambda_1=1$, $\lambda_2=1$ & 0.6 & 0.482353 & 0.54902 \\
\hline
$\lambda_0=0.1$, $\lambda_1=1$, $\lambda_2=0.1$ & 0.6 & 0.488235 & 0.552941 \\
\hline
$\lambda_0=0.1$, $\lambda_1=1$, $\lambda_2=0.01$ & 0.6 & 0.488235 & 0.552941 \\
\hline
$\lambda_0=0.1$, $\lambda_1=0.1$, $\lambda_2=1$ & 0.633333 & 0.417647 & 0.492157 \\
\hline
$\lambda_0=0.1$, $\lambda_1=0.1$, $\lambda_2=0.1$ & 0.6 & 0.482353 & 0.54902 \\
\hline
$\lambda_0=0.1$, $\lambda_1=0.1$, $\lambda_2=0.01$ & 0.6 & 0.488235 & 0.552941 \\
\hline
$\lambda_0=0.1$, $\lambda_1=0.01$, $\lambda_2=1$ & 0.766667 & 0.270588 & 0.4 \\
\hline
$\lambda_0=0.1$, $\lambda_1=0.01$, $\lambda_2=0.1$ & 0.633333 & 0.411765 & 0.490196 \\
\hline
$\lambda_0=0.1$, $\lambda_1=0.01$, $\lambda_2=0.01$ & 0.6 & 0.482353 & 0.54902 \\
\hline
$\lambda_0=0.01$, $\lambda_1=1$, $\lambda_2=1$ & 0.6 & 0.482353 & 0.54902 \\
\hline
$\lambda_0=0.01$, $\lambda_1=1$, $\lambda_2=0.1$ & 0.6 & 0.488235 & 0.552941 \\
\hline
$\lambda_0=0.01$, $\lambda_1=1$, $\lambda_2=0.01$ & 0.6 & 0.488235 & 0.552941 \\
\hline
$\lambda_0=0.01$, $\lambda_1=0.1$, $\lambda_2=1$ & 0.633333 & 0.417647 & 0.492157 \\
\hline
$\lambda_0=0.01$, $\lambda_1=0.1$, $\lambda_2=0.1$ & 0.6 & 0.482353 & 0.54902 \\
\hline
$\lambda_0=0.01$, $\lambda_1=0.1$, $\lambda_2=0.01$ & 0.6 & 0.488235 & 0.552941 \\
\hline
$\lambda_0=0.01$, $\lambda_1=0.01$, $\lambda_2=1$ & 0.766667 & 0.270588 & 0.4 \\
\hline
$\lambda_0=0.01$, $\lambda_1=0.01$, $\lambda_2=0.1$ & 0.633333 & 0.411765 & 0.490196 \\
\hline
$\lambda_0=0.01$, $\lambda_1=0.01$, $\lambda_2=0.01$ & 0.6 & 0.482353 & 0.54902 \\
\hline
Aggregation & 0.8 & 0.441176 & 0.756863 \\
\hline
\end{tabular}
\vspace{5pt}
\caption{Performance metrics for quadratic MP-FR, averaged over $10$ runs.}
\label{tab:quad}
\end{table}

We can make the following important observations:
\begin{enumerate}
\item MP regularisation leads to better results than the single parameter counterpart both for linear and quadratic functional regression.

\item Aggregation is a reliable strategy to address the issue of dealing with multiple regularisation parameters and, especially in the quadratic case, significantly improves performance.

\item In the context of stenosis detection, SE is more important than SP, because it is less dangerous to misdetect a pathology than to misdetect its absence. From this viewpoint, in the present study the aggregation demonstrates an ability to stabilize performance of linear functional regression. Moreover, the results reported in Tables \ref{tab:linear} and \ref{tab:quad} for the aggregation clearly indicate that in terms of SE the quadratic approach outperforms its linear counterpart, that should not be always expected or taken for granted (see, e.g., \cite{horvath2013test}). 
\end{enumerate}

Let us conclude this section by comparing our results with some alternative approaches. The comprehensive survey \cite{kiricsli2013standardized} offers a thorough examination of algorithms designed for detecting stenosis based on vessel cross-section diameters, utilizing the same inputs as our considered methods. While the algorithms discussed in \cite{kiricsli2013standardized} were initially developed for coronary artery stenosis detection, they have the potential applicability to diagnose stenoses in various artery types, including ICA.

It seems that our algorithms demonstrates superior results compared to those reported in \cite{kiricsli2013standardized} (where the stated values were $SE = 0.55$ and $SP = 0.33$). At the same time, we would like to note that the application of polynomial functional regression to the problem of automatic stenosis detection from lumen diameters has been presented here for illustration purposes, and one may expect that nonlinear and non-polynomial functional regression methods may exhibit even better performance.



\section{Acknowledgements}
The research reported in this paper has been supported by the Federal Ministry for Climate Action, Environment, Energy, Mobility, Innovation and Technology (BMK), the Federal Ministry for Digital and Economic Affairs (BMDW), and the Province
of Upper Austria in the frame of the COMET–Competence Centers for Excellent Technologies Programme and the COMET Module S3AI managed by the Austrian Research Promotion Agency FFG. 

The data used in Section \ref{sec:empirical_stenosis} was acquired through the ReSect-study performed at the Medical University of Innsbruck. This study is funded by the OeNB Anniversary Fund (15644).

This work is additionally co-funded by the European Union (ERC, SAMPDE, 101041040). Views and opinions expressed are however those of the authors only and do not necessarily reflect those of the European Union or the European Research Council. Neither the European Union nor the granting authority can be held responsible for them.

We are grateful to the anonymous referees for their valuable suggestions, which have significantly improved our work.

\bibliographystyle{ws-aa}
\bibliography{multi_parameter_functional_aa}

\begin{thebibliography}{10}

\bibitem{aneiros2022functional}
G.~Aneiros, I.~Horov{\'a}, M.~Hu{\v{s}}kov{\'a} and P.~Vieu, On functional data analysis and related topics, {\em Journal of Multivariate Analysis} {\bf 189}  (2022) p. 104861.

\bibitem{bauer2006utilization}
F.~Bauer and S.~Pereverzev, An utilization of a rough approximation of a noise covariance within the framework of multi-parameter regularization, {\em Int. J. Tomogr. Stat} {\bf 4}  (2006)  1--12.

\bibitem{bauer2007regularization}
F.~Bauer, S.~Pereverzyev and L.~Rosasco, On regularization algorithms in learning theory, {\em Journal of complexity} {\bf 23}(1)  (2007)  52--72.

\bibitem{belge2002efficient}
M.~Belge, M.~E. Kilmer and E.~L. Miller, Efficient determination of multiple regularization parameters in a generalized l-curve framework, {\em Inverse problems} {\bf 18}(4)  (2002) p. 1161.

\bibitem{belkin2006manifold}
M.~Belkin, P.~Niyogi and V.~Sindhwani, Manifold regularization: A geometric framework for learning from labeled and unlabeled examples., {\em Journal of machine learning research} {\bf 7}(11)  (2006).

\bibitem{caponnetto2007optimal}
A.~Caponnetto and E.~{De Vito}, Optimal rates for the regularized least-squares algorithm, {\em Foundations of Computational Mathematics} {\bf 7}(3)  (2007)  331--368.

\bibitem{chen2015aggregation}
J.~Chen, S.~Pereverzyev~Jr and Y.~Xu, Aggregation of regularized solutions from multiple observation models, {\em Inverse Problems} {\bf 31}(7)  (2015) p. 075005.

\bibitem{chen2008multi}
Z.~Chen, Y.~Lu, Y.~Xu and H.~Yang, Multi-parameter tikhonov regularization for linear ill-posed operator equations, {\em Journal of computational mathematics}   (2008)  37--55.

\bibitem{dinu2023addressing}
M.-C. Dinu, M.~Holzleitner, M.~Beck, H.~Duc~Nguyen, A.~Huber, H.~Eghbal-Zadeh, B.~Moser, S.~Pereverzyev, S.~Hochreiter and W.~Zellinger, Addressing parameter choice issues in unsupervised domain adaptation by aggregation, in {\em 11 th International Conference on Learning Representations\/},  (OpenReview, 2023).

\bibitem{gizewski2022regularization}
E.~R. Gizewski, L.~Mayer, B.~A. Moser, D.~H. Nguyen, S.~Pereverzyev~Jr, S.~V. Pereverzyev, N.~Shepeleva and W.~Zellinger, On a regularization of unsupervised domain adaptation in rkhs, {\em Applied and Computational Harmonic Analysis} {\bf 57}  (2022)  201--227.

\bibitem{guo2017learning}
Z.-C. Guo, S.-B. Lin and D.-X. Zhou, Learning theory of distributed spectral algorithms, {\em Inverse Problems} {\bf 33}(7)  (2017) p. 074009.

\bibitem{holzleitner2023regularized}
M.~Holzleitner and S.~V. Pereverzyev, On regularized polynomial functional regression, {\em Journal of Complexity} {\bf 83}  (2024) p. 101853.

\bibitem{holzleitner2023domain}
M.~Holzleitner, S.~V. Pereverzyev and W.~Zellinger, Domain generalization by functional regression, {\em Numerical Functional Analysis and Optimization} {\bf 45}(3)  (2024)  259--281.

\bibitem{horvath2013test}
L.~Horv{\'a}th and R.~Reeder, A test of significance in functional quadratic regression, {\em Bernoulli Society for Mathematical Statistics and Probability.} {\bf 19}(5A)  (2013)  2120--2151.

\bibitem{kiricsli2013standardized}
H.~Kiri{\c{s}}li, M.~Schaap, C.~Metz, A.~Dharampal, W.~B. Meijboom, S.-L. Papadopoulou, A.~Dedic, K.~Nieman, M.~A. de~Graaf, M.~Meijs {\em et~al.}, Standardized evaluation framework for evaluating coronary artery stenosis detection, stenosis quantification and lumen segmentation algorithms in computed tomography angiography, {\em Medical image analysis} {\bf 17}(8)  (2013)  859--876.

\bibitem{kokoszka2017introduction}
P.~Kokoszka and M.~Reimherr, {\em Introduction to functional data analysis} (CRC press, 2017).

\bibitem{lin2017distributed}
S.-B. Lin, X.~Guo and D.-X. Zhou, Distributed learning with regularized least squares, {\em The Journal of Machine Learning Research} {\bf 18}(1)  (2017)  3202--3232.

\bibitem{lu2020balancing}
S.~Lu, P.~Math{\'e} and S.~V. Pereverzyev, Balancing principle in supervised learning for a general regularization scheme, {\em Applied and Computational Harmonic Analysis} {\bf 48}(1)  (2020)  123--148.

\bibitem{lu2011multi}
S.~Lu and S.~V. Pereverzev, Multi-parameter regularization and its numerical realization, {\em Numerische Mathematik} {\bf 118}  (2011)  1--31.

\bibitem{lu2013regularization}
S.~Lu and S.~V. Pereverzyev, {\em Regularization theory for ill-posed problems: selected topics}, volume~58 (Walter de Gruyter, 2013).

\bibitem{mayer2019local}
L.~Mayer, C.~Boehme, T.~Toell, B.~Dejakum, J.~Willeit, C.~Schmidauer, K.~Berek, C.~Siedentopf, E.~R. Gizewski, G.~Ratzinger {\em et~al.}, Local signs and symptoms in spontaneous cervical artery dissection: a single centre cohort study, {\em Journal of Stroke} {\bf 21}(1)  (2019) p. 112.

\bibitem{muller2010quadratic}
H.-G. Müller and F.~Yao, Additive modelling of functional gradients, {\em Biometrika} {\bf 97}(4)  (2010)  791--805.

\bibitem{paszke2019pytorch}
A.~Paszke, S.~Gross, F.~Massa, A.~Lerer, J.~Bradbury, G.~Chanan, T.~Killeen, Z.~Lin, N.~Gimelshein, L.~Antiga {\em et~al.}, Pytorch: An imperative style, high-performance deep learning library, {\em arXiv preprint arXiv:1912.01703}   (2019).

\bibitem{pereverzyev2022introduction}
S.~Pereverzyev, {\em An Introduction to Artificial Intelligence Based on Reproducing Kernel Hilbert Spaces} (Springer Nature, 2022).

\bibitem{ramsay1982data}
J.~O. Ramsay, When the data are functions, {\em Psychometrika} {\bf 47}  (1982)  379--396.

\bibitem{ramsay1991some}
J.~O. Ramsay and C.~J. Dalzell, Some tools for functional data analysis, {\em Journal of the Royal Statistical Society Series B: Statistical Methodology} {\bf 53}(3)  (1991)  539--561.

\bibitem{ramsay2002applied}
J.~O. Ramsay and B.~W. Silverman, {\em Applied functional data analysis: methods and case studies} (Springer, 2002).

\bibitem{reiss2017methods}
P.~T. Reiss, J.~Goldsmith, H.~L. Shang and R.~T. Ogden, Methods for scalar-on-function regression, {\em International Statistical Review} {\bf 85}(2)  (2017)  228--249.

\bibitem{tao2014polynomial}
Z.~Tao, Z.~Qingzhao and W.~Qihua, Model detection for functional polynomial regression, {\em Computational Statistics and Data Analysis} {\bf 70}(4)  (2014)  83--197.

\bibitem{tong2021distributed}
H.~Tong, Distributed least squares prediction for functional linear regression, {\em Inverse Problems} {\bf 38}(2)  (2021) p. 025002.

\bibitem{tong2018analysis}
H.~Tong and M.~Ng, Analysis of regularized least squares for functional linear regression model, {\em Journal of Complexity} {\bf 49}  (2018)  85--94.

\bibitem{wang2016functional}
J.-L. Wang, J.-M. Chiou and H.-G. M{\"u}ller, Functional data analysis, {\em Annual Review of Statistics and its application} {\bf 3}  (2016)  257--295.

\bibitem{xu2006multiple}
P.~Xu, Y.~Fukuda and Y.~Liu, Multiple parameter regularization: numerical solutions and applications to the determination of geopotential from precise satellite orbits, {\em Journal of Geodesy} {\bf 80}  (2006)  17--27.

\bibitem{yuan2010reproducing}
M.~Yuan and T.~T. Cai, A reproducing kernel hilbert space approach to functional linear regression, {\em Annals of Statistics} {\bf 6}(38)  (2010)  3412 -- 3444.

\bibitem{yurinsky1995sums}
V.~Yurinsky, {\em Sums and Gaussian Vectors.} (Lecture Notes in Mathematics. Springer Berlin, Heidelberg, 1995).

\end{thebibliography}

\end{document}